 \documentclass[final,3p,times]{elsarticle}

\usepackage{amssymb}

\usepackage{verbatim}
\usepackage{amsmath}
\usepackage{amsthm}
\usepackage[ruled,vlined]{algorithm2e}
\usepackage{subfigure}
\usepackage{graphicx}   
\usepackage{epstopdf} 
\usepackage[dvipsnames]{xcolor}
\usepackage[gen]{eurosym}
\usepackage{url}     
\usepackage{float}  
\usepackage{hyperref}
\usepackage{cite}

\newcommand{\ejp}[1]{{\color{red}#1 }}

\newcommand{\given}{\,\,|\,\,}
\newcommand{\n}{|\!|}
\newcommand{\R}{\mathbb{R}}
\newcommand{\tr}{\mbox{Tr}}

\usepackage{ragged2e}
\usepackage{amsmath}
\theoremstyle{plain}
\newtheorem{theorem}{Theorem}

 \usepackage{lineno}
\journal{Ambient Intelligence and Humanized Computing}
\graphicspath{{figures/}}

\author[cwi,tud]{Abdolrahman Khoshrou\corref{cor1}\fnref{fn1}}
\ead{a.khoshrou@cwi.nl}
\author[cwi]{Eric J. Pauwels\fnref{fn2}} %

\cortext[cor1]{Corresponding author: \textit{a.khoshrou@cwi.nl}}

\fntext[fn1]{Abdolrahman Khoshrou is a (guest) PhD student at TU Delft. He is also with the Intelligent and Autonomous Systems Group at National Dutch Mathematics and Informatics Center (CWI).}

\fntext[fn2]{Eric Pauwels is a senior researcher and the leader of the Intelligent and Autonomous Systems Group at CWI.}

\address[cwi]{Centrum Wiskunde \& Informatica, Science Park 123, 1098 XG, Amsterdam, The Netherlands}
\address[tud]{Department of Mathematics and Computer Science, Delft University of Technology, The Netherlands}

\begin{document}
\begin{frontmatter}
\title{
Regularisation for PCA- and SVD-type 
matrix factorisations 
}
\begin{abstract}
Singular Value Decomposition (SVD) and its close relative, Principal Component Analysis (PCA), are well-known linear matrix decomposition techniques that are widely used in applications such as dimension reduction and clustering. 
However, an 
important limitation of  SVD/PCA is its sensitivity to 
noise in the input data. In this paper, we take another look at the problem 
of regularisation and show that different 
formulations of the minimisation problem  lead to 
qualitatively different solutions. 


\end{abstract}
\end{frontmatter}




\section{Introduction and Motivation} \label{sec:intro}
\subsection{Introduction and Related Work}
Singular Value Decomposition (SVD)
 and its close relative, Principal Component Analysis (PCA), 
are well-known linear matrix factorisation  techniques that are widely used in applications as varied as dimension reduction 
and clustering, 
matrix completion~\citep{Davenport_2016} (e.g. for recommender systems),  dictionary learning~\citep{tovsic2011dictionary} and time series analysis~\citep{khoshrou2018data}.
In a surprising turn of events, (deep) matrix factorisation 
also plays a role in the implicit regularisation that 
enables 
acceptable generalisation 
in deep learning~\citep{gunasekar2017implicit}.

In their abstract version, SVD and PCA  amount to two different 
but related types of matrix factorisation.  
More precisely, given a general (data) matrix $A$, the aim is 
to approximate it as a product of simpler 
(i.e. lower-rank) matrices. 
Specifically: 
\begin{itemize}
    \item PCA-type decomposition: $A \approx P Q^T $  
    where  the columns of $Q$ are orthonormal, i.e. $Q^T Q = I$;
    \item SVD-type decomposition: $ A \approx  PBQ^T $
    where $B$ is diagonal, while $P^TP = I, \, \,Q^TQ =I$.
\end{itemize}
The approximation in the above equations is 
measured in terms of the Frobenius (matrix) norm which 
for an arbitrary matrix $X \in \R^{n\times m}$ is 
defined as: 
\begin{equation}
\n X  \n_F^2 = \sum_{i=1}^n \sum_{j = 1}^m  x_{ij}^2 = \tr(XX^T) = \tr(X^TX) = \n X^T  \n_F^2 .
\label{eq:Frobenius}
\end{equation}
(In the remainder of the paper, we will drop the subscript $F$).

%

Although these factorisation techniques are both conceptually 
simple and effective, it is well-known that they are 
sensitive to noise and outliers in the input data. 
As a consequence,  some modifications of the original 
algorithms have been proposed to alleviate the effect of 
these disturbances~\citep{brooks2013pure,kwak2013principal}. 
Candes et al.~\citep{candes2011robust} introduce  {\it Robust PCA (RPCA)} which aims to 
separate signal from outliers by decomposing  any given matrix into the sum of a low-rank approximation and a sparse matrix of outliers.
An extension of this work for inexact recovery of the data is presented in~\citep{zhou2010stable}. 
Another example of sparse PCA using low rank approximation is proposed in~\citep{shen2008sparse}. 

Adding a regularisation term is another 
versatile  
way to tackle the problem of noisy input.  For instance,  
Dumitrescu et al.~\citep{dumitrescu2017regularized} 
show how a regularized version of K-SVD algorithm can be adapted to the Dictionary Learning (DL) problem. 
However, the presence of noise in the input is not the only 
reason to invoke regularisation.
Recent research~\citep{jin2015low} shows that 
in many real data sets, it is  
not only the observed data that lie on a (non-)linear low dimensional manifold, but this 
also applies to the features.
He et al.~\citep{he2019graph} 
point out that if the columns of the matrix $A$ 
are interpreted as data points, then the rows are features. 
The neighbourhood structure of both the data points 
and the features give rise to distinct graphs (the 
so-called data 
and the feature graph) and hence,  to corresponding 
graph 
Laplacians ($L_d$ and $L_f$ respectively).  
The resulting regularised PCA is referred to as 
the {\it  graph-dual Laplacian PCA }
(gDLPCA) and for a given 
data matrix $A$, is obtained 
by minimising the functional: 

\begin{equation}
J(V,Y) = \n A - VY\n^2 
     + \alpha \, \tr(V^T L_d V) 
     + \beta \, \tr(YL_f Y^T) 
     \quad \quad \mbox{ subject to \,} V^TV = I
     \label{eq:J_He_1}
\end{equation}
The ability of the 
graph dual regularization technique
to incorporate both data and feature structure
has deservedly attracted considerable attention in 
dimensionality 
reduction applications~\citep{yin2015dual,shahid2016fast, he2019graph}.

In the present paper, we 
take the functional.~\eqref{eq:J_He_1} as a starting 
point and investigate the two factorisation approaches 
mentioned above (invoking eq.~\eqref{eq:Frobenius} 
to recast the trace as a norm): 

\begin{itemize}
    \item PCA-type decomposition ($A \approx PQ^T$) by minimising the regularisation  functional:
    
    \begin{equation}
     \n A - PQ^T\n^2 + \lambda\, \n DP \n^2  + \mu\,  \n GQ \n^2 
   \label{eq:pca_fac}
   \end{equation}
    
    \item SVD-type decomposition ($ A \approx  PBQ^T $) 
    by minimising the regularisation functional:
     \begin{equation}
      \n A - PBQ^T\n^2 + \lambda\, \n DP \n^2  + 
   \mu\, \n GQ \n^2 
   \label{eq:svd_fac}
   \end{equation}
  
\end{itemize}
The minimisation of the  functional~\eqref{eq:pca_fac} 
was discussed in \citep{he2019graph}, but the proposed 
solution contains an error which we correct in this paper. 
In addition, we also provide an algorithm to solve functional~\eqref{eq:svd_fac}, which somewhat surprisingly 
is quite different from the one for~\eqref{eq:pca_fac}.

The remainder of this paper is organised as follows: 
We finalise this section by recapitulating some important 
facts facts about SVD. 
In section~\ref{sct:pca_type_reg} 
and \ref{sct:svd_type_reg} we derive an algorithm 
for minimisation of the regularised version of PCA-type 
and SVD-type
factorisation, respectively. 
In section~\ref{sct:comp_aspects} how gradient descent 
can be implemented by drawing on some elementary facts 
from Lie-group theory.  Finally, we conclude 
by giving  some pointers to potential extensions.


\subsection{Brief recap of Singular Value Decomposition (SVD)}
For the sake of completeness, we first recall the well-known  SVD result; for more details we refer to 
standard textbooks such as~\citep{strang1993introduction}\citep{Horn1985}.
%
\begin{theorem}[{\bf Singular Value Decomposition, SVD}] 
\label{thm:svd_1}
Any real-valued $n\times m$  matrix $A$ can be factorized into the product of three matrices: 
\begin{equation}
   A = U S V^T   \quad \quad \mbox{where}   \quad 
   U \in {\cal O }(n) \quad  \mbox{and} \quad  V  \in {\cal O }(m)  \quad 
   \mbox{are orthonormal, }
   \label{eq:svd_def}
\end{equation}
and $S$ is an $n\times m$ diagonal matrix where the elements on the main
``diagonal'' (so-called {\it singular values} ) are non-negative (i.e. $\sigma_i:=S_{ii} \geq 0$ for $1 \leq i \leq 
\min(n,m)$). \\
Assuming that the rank  $rk(A) = r \leq \min(n,m)$,  we can sort the 
singular values such that
$$\sigma_1 \geq \sigma_2 \geq \ldots \geq \sigma_r > 0 = 
\sigma_{r+1} = \ldots = \sigma_{\min(n,m)}$$
and recast eq.~\eqref{eq:svd_def} as 
\begin{equation}
     A  = \sum\limits_{i = 1}^r \sigma_i U_i V_i^T    \quad \quad 
     \mbox{where $U_i, V_i$ are the $i$-th columns of $U$ and $V$, 
     respectively.}
     \label{eq:svd_def_2}
\end{equation}
For the singular values sorted as above, 
we introduce the short-hand 
notation $U_{(1:k)}$  and $V_{(1:k)}$ to denote the matrix 
comprising the first $k$ {\it columns} of $U$ and $V$, 
respectively: 
$$
U_{(1:k)}  :=  [U_1, U_2, \ldots, U_k] \quad\quad 
\mbox{and} \quad \quad 
V_{(1:k)}  :=  [V_1, V_2, \ldots, V_k] .
$$
In this notation, eq.~\eqref{eq:svd_def_2} can be 
expressed concisely as: 
\begin{equation}
    A = U_{(1:r)} \> diag(\sigma_1, \ldots , \sigma_r) 
    \> V_{(1:r)}^T. 
    \label{eq:svd_def_3}
\end{equation}

\qed
\end{theorem}
To appreciate the significance of Theorem~\ref{thm:svd_1}, it 
is helpful to   
highlight its geometric interpretation. 
 Recall that any  $n\times m$ matrix $A$ gives rise to a 
 corresponding linear transformation 
  $ A:  \R^m \longrightarrow \R^n  $
    that maps the standard basis in $\R^m$ into the columns of $A$:
    $$   A \mathbf{e}_k  = A_k  \quad \quad 
    \mbox{where} \quad \mathbf{e}_k = {(0, 0, \ldots,0, 1,0,  \ldots, 0)^T}.
    $$
    Roughly speaking, the SVD theorem therefore tells us that it is always possible to select  
    an {\it orthonormal } basis in $\R^m$ (columns of $V$) that is 
    mapped (up to non-negative scaling factors, i.e. 
    the singular values) into an 
    {\it orthonormal} basis in $\R^n$ (columns of $U$). This is immediately 
    obvious from eq.~\eqref{eq:svd_def_2}:
    
    $$ A V_{\ell}  =  \sum\limits_{k = 1}^r \sigma_kU_k V_k^T   V_{\ell}   
    = \sum\limits_{k = 1}^r \sigma_k U_k \delta_{k\ell}  
    =  \sigma_{\ell} U_{\ell}. 
    $$
where $\delta_{k \ell }$ is a Kronecker delta function.  
It is worth noting that insisting on the orthogonality of $V$ ($V^TV = I$)
is not restrictive. 
Indeed, a linear transformation is completely and uniquely determined by specifying its 
effect on any basis, and there is no loss of generality 
by insisting on the orthonormality of this basis. 
However, the non-trivial message of this theorem is  
this  orthonormal basis ($V$) can be chosen 
in such a way that its image $U$ under $A$ is also orthonormal 
(again, up to non-negative scalings). 
Furthermore, 
in a generic case (where all singular values are different)   
the singular value decomposition is unique, up to an arbitrary 
relabeling of the {\it basis-vectors} and a simultaneous sign-flip 
of corresponding columns in $U$ and $V$, i.e. 
$(U_\ell,V_\ell) \rightarrow (-U_\ell,-V_\ell) $  for any number of columns. 

The importance of the SVD result, and the starting point for this paper, is the following well-known minimisation result (more details can be found in ~\citep{golub2013matrix,eckart1936approximation}).

\begin{theorem}[{\bf Eckart-Young-Mirsky Theorem:  Optimal low rank approximation}]
\label{thm:rank_approx}
Let us consider an $n \times m$ matrix $A$ with rank $rk(A) = r \leq \min(n,m)$. 
For $k < r$,  finding the rank-$k$ matrix $A_k$ that is closest to $A$ in (Frobenius) norm gives rise to the following constrained minimisation problem: 
\begin{equation*}
    \min_{A_k} \n A - A_k\n^2    \quad \quad 
\mbox{subject to}  \quad rk(A_k) \leq  k.
\end{equation*}
The solution to this problem 
is obtained by truncating the SVD expansion eq.~\eqref{eq:svd_def_2} 
after the $k$-th largest singular value: 
\begin{equation}
     A_k  = \sum\limits_{i= 1}^k \sigma_i U_i V_i^T  
      =  U_{(1:k)} \> diag(\sigma_1, \ldots , \sigma_k) 
    \>  V_{(1:k)}^T. 
\end{equation}
\qed
\end{theorem}
Recall that a rank-$k$ matrix of size $n\times m$ can 
always be written as a product $A_k = PQ^T$ where 
$P \in \R^{n \times k}$ and 
$Q \in \R^{m \times k}$ are matrices of full rank $k$. 
Again, in this factorisation, there is no loss of 
generality in requiring $Q^T Q = I_k$.  In fact, 
it is necessary to remove indeterminacy due to arbitrary but trivial rescalings such as  $P \longmapsto rP $ 
while $Q \longmapsto (1/r)Q$ (with $r\neq 0$), and the like.
Hence, one can reformulate Theorem~\ref{thm:rank_approx} as the  factorisation result in Theorem~\ref{thm:low_rank_factors}.
\begin{theorem}[{\bf PCA-type factorisation}]
\label{thm:low_rank_factors}
Assume that the $n \times m$ matrix $A$ has rank $rk(A) = r \leq \min(n,m)$. 
We now define the functional $G(P,Q)$ as follow:
\begin{equation} 
G(P,Q) = 
\n A - PQ^T \n^2
\label{eq:low_rank_factor_0}
\end{equation}
and the corresponding constrained optimisation problem:
\begin{equation} 
\min_{P, Q} G(P,Q)   \quad \quad 
\mbox{subject to}  \quad rk(P) = rk(Q) = k \quad 
\mbox{and} \quad  Q^TQ  = I_k
\label{eq:low_rank_factor_1}
\end{equation}
where $k < r$. A solution to the above constrained minimisation problem 
(in $P \in \R^{n \times k}$ and 
$Q \in \R^{m \times k}$)
is given by (using the SVD notation given above): 
\begin{equation} 
Q = V_{(1:k)}  \quad\quad 
\mbox{and } \quad \quad 
P = U_{(1:k)}\> diag(\sigma_1, \ldots , \sigma_k)
\label{eq:low_rank_factor_2}
\end{equation}
hence: 
\begin{equation}
PQ^T = \sum\limits_{i=1}^k \sigma_i U_i V_i^T.
\label{eq:low_rank_factor_3}
\end{equation}
From \eqref{eq:low_rank_factor_2} 
this it also follows that  $P^TP$ is diagonal, but not 
necessarily equal to the identity. 
\qed
\end{theorem}
Note that If we drop the insistence on the diagonal form 
    for $P^T P$  (i.e. $P$ need no longer be an orthogonal frame), then 
the solution is no longer unique.  Indeed, by taking any $ k \times k$ orthogonal matrix $R$ with 
$R^T R  = I_k = RR^T$, it is clear that $P' = PR$  
and $Q' = QR$ are also solutions. In this case: 
$Q'^T Q' = R^TQ^TQ R = I_k$ but $P'^T P' = R^T P^T P R 
= R^T (SS^T) R$ is in general a positive definite symmetric matrix.


\section{Regularisation for PCA-type 
factorisation}
\label{sct:pca_type_reg}

%

\subsection{Regularised PCA}
The following theorem outlines an obvious 
generalisation to the regularised version of the minimisation problem.
\begin{theorem}[{\bf Regularised PCA}]
\label{thm:rank_approx_reg}
Let $A$ be an $n\times m$ matrix of rank $r \leq \min(n,m)$.  For $k \leq r$,  let $P \in \R^{n\times k}$ and $Q \in \R^{m\times k} $ full rank matrices (i.e. of rank $k$).  Furthermore, 
for arbitrary (non-zero) integers $d$ and $g$
we introduce 
regularisation matrices $D \in \R^{d\times n}$ 
and $G \in \R^{g\times m}$,   
as well as  weights $\lambda,\mu \geq 0$. 

We now define the following functional $F$ in the 
variables $P$ and $Q$:
\begin{equation}
F(P,Q) =  \n A - PQ^T\n^2 + \lambda\, \n DP \n^2  
  + \mu\, \n GQ \n^2 
\label{eq:functional_factor_full} 
\end{equation}
and pose the corresponding {constrained} optimisation problem: 
\begin{equation}
\min_{P,Q}F(P,Q) \quad \quad 
\mbox{subject to} \quad \quad Q^TQ = I_k.
\label{eq:functional_2_lambda_mu} 
\end{equation}
Introducing short-hand notation 
$ L := D^T D \in \R^{n\times n}$ and 
$M := G^T G  \in \R^{m\times m}$ (both symmetric and positive semi-definite), the solution of the constrained 
optimisation problem \eqref{eq:functional_2_lambda_mu} is constructed as follows: 
\begin{itemize}
    \item The $k$ columns of the $m\times k$ matrix $Q$ are the eigenvectors of the $m\times m$ matrix:
    $$ K:= A^T(I_n + \lambda L)^{-1}  A  -  \mu \,M$$
    corresponding to the $k$ largest eigenvalues;
    \item Furthermore:  $P = (I_n + \lambda L)^{-1}  AQ$
\end{itemize}
\end{theorem}
For the sake of completeness, we reiterate that the condition 
$Q^T Q = I_k$ is not restrictive but necessary to eliminate 
arbitrary rescalings. 
In passing, we point out that result above corrects an error 
in \citep{he2019graph} where it is incorrectly stated that $P = AQ$.
%
%
\begin{proof}
\label{sct:main_proof}

Since the variable $P$ in the 
functional~\eqref{eq:functional_factor_full} 
in unconstrained, we can identify the optimum 
in $P$ (for fixed $Q$) by computing 
the gradient: 
\begin{eqnarray}
{\displaystyle \frac{1}{2} \nabla_P  F} &=& (PQ^T -A)Q + \lambda D^TD P
\end{eqnarray}
and solving for $P$:
\begin{equation}
\nabla_P F =0 \quad \Rightarrow \quad  P \underbrace{Q^TQ}_{I_k} - AQ + \lambda L P=0  \quad \Rightarrow \quad
    (I_k + \lambda L)P = AQ  .
    \label{eq_solution_P}
\end{equation}
This condition needs to hold at the solution point. 
By first re-writing $F(P,Q)$ formula as the trace of matrices and then plugging in~\eqref{eq_solution_P}, we have:
\begin{eqnarray*}
F(P,Q) &=& \tr\left[ (A-PQ^T)(A^T-QP^T) \right] + \lambda \> \tr(P^T{L}P) +  \mu \,\tr(Q^T M Q)\\
&=&\tr\, \left[ 
AA^T- AQP^T -P Q^TA^T +P Q^TQ P^T \right] + \lambda \> \tr (P^TLP)
+  \mu \,\tr(Q^T M Q)
\end{eqnarray*}
Considering the fact that the trace operator is invariant under transposition 
as well as cyclic permutation, 
and plugging in eq.~\eqref{eq_solution_P} we arrive at: 
\begin{eqnarray}
F(P,Q) &=& \tr\left[ AA^T -2(I_n+\lambda L)PP^T+PP^T \right] + \lambda \> \tr(P^TLP) +  \mu \,\tr(Q^T M Q) \nonumber \\
&=& \tr \left( AA^T - PP^T -2 \lambda L PP^T \right) + \lambda \> \tr(P^TLP) +  \mu \, \tr(Q^T M Q) \nonumber\\
&=& \tr(AA^T) - \tr(PP^T) -2 \lambda \> \tr(LPP^T) + \lambda \> \tr(P^TLP) +  \mu \,\tr(Q^T M Q) \nonumber \\
 &=& \tr(AA^T) - \tr(P^TP) -\lambda \> \tr(P^TLP)
    +  \mu \,\tr(Q^T M Q) \nonumber\\
    &=& \tr(AA^T) - \tr\left[ P^T \underbrace{(I_n+\lambda L)P}_{AQ}\right] +  \mu \,\tr(Q^T M Q).
\end{eqnarray}
Extracting $P$ and its transpose from eq.~\eqref{eq_solution_P}:  
\begin{equation}
    P=(I_n+\lambda L)^{-1}AQ \quad \Rightarrow  \quad 
    P^T = Q^TA^T (I_n+\lambda L)^{-1}\quad \quad \text{ as $L$ is symmetric}
    \label{eq_P}
\end{equation}
we arrive at:
\begin{equation}
    F(P,Q) = \tr(AA^T)-\tr\left[ Q^T \left(A^T(I_n+\lambda L)^{-1}A - \mu M \right)Q\right].
\end{equation}
Therefore, in order to minimize $F$, one must maximize the right-most term 
as $\tr(AA^T)$ is a constant.  
This is achieved by selecting for $Q$, eigenvectors corresponding to the 
$k$ largest eigenvalues of 
$(A^T(I_n+\lambda L)^{-1}A -\mu M)$. 
Once $Q$ is determined, $P$ is obtained via eq.~\eqref{eq_P}. 

As a concluding remark, we  point out that the matrix 
$I_n + \lambda L$ is always  invertible.
Indeed, since $L=D^TD$ is positive semi-definite and symmetric, it has a complete set of eigenvectors 
with corresponding non-negative eigenvalues, i.e., $L=W\Lambda W^T$, where $W$ is orthogonal 
(i.e. $W^T W = WW^T = I_n$) 
and $\Lambda \geq 0$. 
Hence, the matrix
$(I_n + \lambda L)$
has strictly positive diagonal elements, 
and is indeed invertible. 
\end{proof}
Some illustrative numerical experiments can be found~\citep{code_theorem_4}.

\subsection{Some special cases}
\label{sct:some_special_cases}

\begin{itemize}
    \item \framebox{$\lambda = 0$ and $\mu =0$}:\quad In that case, $Q$ comprises the first $k$ eigenvectors of 
    $  K = A^TA $ and  $P = AQ$, which means that 
    we end up with the standard SVD, as expected. Some 
    numerical experiments can be found~\citep{code_special_case_mu_0_lambda_0}. 
    
    
    \item \framebox{$D = I_n$ and $\mu = 0 $}:\quad The following section provides an overview of the results in~\citep{dumitrescu2017regularized} where a regularized K-SVD problem is addressed. 
    In the aforementioned paper, the authors consider a special case, where  $\mu = 0 $ and $D = I_n$.  Since this 
implies that $L = D^T D = I_n$ and $\mu M = 0$, the matrix $K$ simplifies 
to 
$$ K = \frac{1}{1+\lambda}\,A^TA $$
The eigenvectors of $K$ 
are therefore the right singular 
vectors of $A$ (i.e. the eigenvectors of $A^TA$). 
Hence $Q = V_{(1:k)}$, and as a result: 
$$  P = \frac{1}{1 + \lambda}AQ\quad and\quad AQ = U_{(1:k)}\, diag(\sigma_1, \ldots , \sigma_k) . $$

In particular, for $k=1$ (the rank-$1$ reconstruction), we obtain: 
$$ Q = \mathbf{v}_1  \quad \quad \mbox{and} \quad\quad 
P = \frac{\sigma_1}{1+\lambda} \,\mathbf{u}_1$$
which is the result that can be found in \citep{dumitrescu2017regularized}. The experiments are available in~\citep{code_special_case_mu_0_D_In}.


\end{itemize}

\section{Regularisation for SVD-type 
factorisation}
\label{sct:svd_type_reg}
We now turn our attention to the SVD-type factorisation  
which looks for an approximation of the form: 
\begin{equation*}
   A \approx PBQ^T  \quad\quad 
\mbox{subject to:}  \quad 
 \quad   Q^TQ = I_k, \quad  \,\n P_i\n=1 \,\,
 \mbox{$\forall i\in \{1,2,\ldots, k\}$, \, and  $B$ diagonal.} 
\end{equation*}
Loosely speaking, since the columns of $P$ and $Q$ are 
of unit length, they only pins down the structure of 
$A$, whereas  the diagonal matrix $B=diag(\beta_1, \beta_2, \ldots, \beta_k)$ captures the \emph{amplitude} 
of the corresponding structures. 
Similar to before, the columns of $Q$ are orthonormal, i.e., we again insist on $Q^TQ = I_k$.
However, unlike before, the columns of $P$ are now 
only required to have unit length. 
%

In light of the aforementioned SVD-type matrix factorisation technique, Theorems~\ref{thm_rsvd_mu_0} and~\ref{thm_rsvd_mu_not_0} provide an alternative solution to the lower-rank matrix approximation problem.
For notational convenience,  
Theorem~\ref{thm_rsvd_mu_0} first addresses the 
simplified case for $\mu=0$. 
Finally, in 
Theorem~\ref{thm_rsvd_mu_not_0} 
we return to the general case. 
\begin{theorem}[{\bf Regularised SVD}]
\label{thm_rsvd_mu_0}
Let $A$ be an $n\times m$ matrix of rank $r \leq \min(n,m)$.  For $k \leq r$,  let $P \in \R^{n\times k}$ and $Q \in \R^{m\times k} $ of rank $k$, 
while $B \in \R^{k\times k} $ diagonal 
(i.e.  $ B = diag(\beta_1, \beta_2, \ldots, \beta_k)  $).  Furthermore, 
for arbitrary non-zero integer $d$ 
we introduce 
regularisation matrix $D \in \R^{d\times n}$, 
as well as  weight $\lambda \geq 0$. 
Finally, we introduce the short-hand notation 
$ L := D^T D \in \R^{n\times n}$ 
%
(symmetric and positive-definite).
We are now in a position to define the following functional $F$ in the 
variables $P, Q$ and $B$:
\begin{equation}
F(P,Q, B) =  \n A - P B Q^T\n^2 + \lambda\, \n DP \n^2,   
\label{eq:functional_1} 
\end{equation}
and the corresponding constrained optimisation problem: 
\begin{equation}
\min_{P,Q, B}F(P,Q, B) \quad \quad 
\mbox{subject to:}  \quad 
 \quad Q^TQ = I_k, \quad  \,\n P_i\n=1 \,\,
\mbox{$\forall i\in \{1,2, \ldots,k\}$, \, and  $B$ diagonal.} 
\label{eq:functional_rsvd} 
\end{equation}
This problem is solved by the solution Algorithm~\ref{algorithm_mu_0}
specified below.  
\\
\begin{algorithm}[H]
\SetAlgoLined
\KwIn{$A,\> k,\> \lambda, \> D $}
\KwOut{$P,\> B,\> Q $}
 Initialization\\
 \While{no convergence}{
\begin{enumerate}
    \item Determine the $m\times k$ matrix $Q = [\mathbf{q}_1, \mathbf{q}_2, \ldots, \mathbf{q}_k]$ 
    (with orthonormal columns: $Q^TQ = I_k$) 
    such that the sum of the 
    smallest eigenvalue of each of the $k$ symmetric matrices 
    $S(\mathbf{q}_i) $ is minimal, i.e.:
    $$ \min_{Q} \psi(Q) =  \min_{Q} \sum_{i=1}^k  \lambda_1(\mathbf{q}_i) \quad\quad \mbox{such that } 
    Q^TQ=I_k$$
    
    where $\lambda_1(\mathbf{q}_i) = \min (eig(S(\mathbf{q}_i))$. To this end we use gradient descent (see Section~\ref{sct:comp_aspects}). 
    \item  For each $\mathbf{q}_i$ as determined above, take $\mathbf{p}_i$  
    to be the eigenvector $W_1(\mathbf{q}_i)$ corresponding 
    to the smallest eigenvector $\lambda_1(\mathbf{q}_i)$. 
    Construct the $n\times k$ matrix 
     $P = [\mathbf{p}_1, \mathbf{p}_2, \ldots, \mathbf{p}_k]$. 
    \item Finally, set $B = diag(\beta_1, \ldots, \beta_n)$ where  $\beta_i = (P^TAQ)_{ii}$.
\end{enumerate}
 }
 \caption{Proposed RSVD method ($\mu = 0$)}
 \label{algorithm_mu_0}
\end{algorithm}

    
    
    
    
    

\end{theorem}
\begin{proof}
Since $B$ is unconstrained, we can determine its optimal value 
by computing the derivative with respect to $B$ and 
equating it to zero:
\begin{equation}
\nabla_B F(P,Q,B) = \nabla_B  \n A - P B Q^T\n^2 . 
\label{eq:nabla_B_1}
\end{equation}
Expanding the norm in terms of a trace 
(cf. eq.~\eqref{eq:Frobenius}), 
and using the invariance of a trace under transposition, 
we arrive at (recall $Q^TQ = I_k$ ):
\begin{eqnarray} 
\n A - P B Q^T\n^2 & = & 
\tr\left[ (A - P B Q^T)(A^T - Q B P^T) \right]   \nonumber \\
&=& \tr (AA^T) - 2 \tr(AQBP^T)  + \tr(P B^2P^T)  \nonumber \\[1ex]
&=& \n A \n^2   - 2\, \mbox{\tr}(P^TAQB)  
+ \mbox{\tr}(B^2 P^TP)
\nonumber \\
& = & \n A \n^2   - 2\,\sum_{i=1}^k (P^TAQ)_{ii}\, \beta_i  
+ \sum_{i=1}^k (P^TP)_{ii}\,\beta_i^2 \quad \quad \text{($B$ is diagonal)} \\
& = & \n A \n^2   - 2\,\sum_{i=1}^k (P^TAQ)_{ii}\, \beta_i  
+ \sum_{i=1}^k \beta_i^2 \quad \quad \text{($\n P_i \n=1 \Rightarrow  (P^TP)_{ii} = 1$).}
\label{eq:functional_beta}
\end{eqnarray}
We therefore calculate the gradient of the functional $F$ with respect to $B$ as follow:
$$
\frac{\partial}{\partial \beta_i}\,  \n A - P B Q^T\n^2  = 
2\, ( \beta_i  -  (P^TAQ)_{ii}).  
$$
For given $P$ and $Q$, we find the optimal $B$ by insisting that 
the resulting gradient vanishes, which yields: 
\begin{equation}
    \beta_i = (P^TAQ)_{ii}  \quad \quad \forall i \in \{1,2, \ldots, k\}.
    \label{eq:beta_optimal}
\end{equation}
Plugging this optimal choice  back into eq.~\eqref{eq:functional_beta}
the functional~\eqref{eq:functional_1} simplifies to 
\begin{equation}
\n A - P B Q^T\n^2  =  \n A\n^2 - \sum_{i = 1}^k \beta_i^2 
\label{eq:beta_2}
\end{equation}
To recast eq.~\eqref{eq:beta_2} in terms of $P$ and $Q$ 
(in order to eliminate $B$), we observe that for an arbitrary matrix $H$ 
we have 
$ H_{ij} = \mathbf{e}_i^T H  \mathbf{e}_j $,   where 
$\mathbf{e}_i = (0,0,\ldots, 1, \ldots, 0)^T$ are the standard basis 
vectors. 
Hence,  using 
the fact 
that the diagonal of a matrix is unchanged under 
transposition,  we conclude that 
$$ \beta_i = 
\left\{  \begin{array}{lcl}
(P^TAQ)_{ii} &=&  \mathbf{e}_i^T P^T A Q\,\mathbf{e}_i = 
\mathbf{p}_i^T  A \mathbf{q}_i\\[2ex]
(Q^TA^TP)_{ii} &=&  \mathbf{e}_i^T Q^T A^T P \,\mathbf{e}_i = 
\mathbf{q}_i^T  A^T \mathbf{p}_i
\end{array}
\right.
$$
where $\mathbf{p}_i, \mathbf{q}_i$ are the $i$-th columns of $P$ and 
$Q$, respectively.  
i.e. 
 $P = [\mathbf{p}_1, \mathbf{p}_2, \ldots, \mathbf{p}_k]$  and 
$Q = [\mathbf{q}_1, \mathbf{q}_2, \ldots, \mathbf{q}_k]$. 
As a consequence, 
\begin{equation}
\sum_{i = 1}^k \beta_i^2 =    
\sum_{i = 1}^k  \mathbf{p}_i^T  A \mathbf{q}_i \, \mathbf{q}_i^T  A^T \mathbf{p}_i. 
\label{eq:beta_3}
\end{equation}
As a final step, we introduce the notation  $L = D^T D$ to recast the 
regularisation term as: 
\begin{equation}
    \n DP \n^2 = \tr(P^T L P) 
    = \sum_{i=1}^k  \mathbf{e}_i^T P^T L P\, \mathbf{e}_i 
    = \sum_{i=1}^k \mathbf{p}_i^T  L \, \mathbf{p}_i.
\label{eq:reg_1}    
\end{equation}
%
Plugging eqs.~(\ref{eq:beta_3}) and (\ref{eq:reg_1}) 
into eq.~(\ref{eq:functional_1}), we obtain the following simplified 
form for the functional $F$ (assuming that we eliminate $B$ 
by using its optimal value): 

\begin{equation}
F(P,Q) = \n A\n^2 + F_1(P,Q), \quad\quad 
\mbox{where} \quad \quad 
 F_1(P,Q) =
\sum_{i=1}^k  \mathbf{p}_i^T(\lambda L - A\mathbf{q}_i\mathbf{q}_i^TA^T)\mathbf{p}_i. 
    \label{eq:functional_2}
\end{equation}
Introducing the notation $ S(\mathbf{q}) := \lambda L  - A\mathbf{q}\mathbf{q}^TA^T  $, 
we conclude that  
$$ F_1(P,Q) = \sum_{i=1}^k  \mathbf{p}_i^T 
S(\mathbf{q}_i)\mathbf{p}_i. $$
Since each $S(\mathbf{q})$ is a symmetric matrix, it 
can be diagonalised with respect to  
an orthonormal basis, i.e. 
there is an orthogonal $n \times n $ 
matrix $W$ (with $W^T W = WW^T = I_n$) 
and a diagonal matrix 
$\Lambda = diag(\lambda_1, \ldots, \lambda_n)$  (ordered 
$\lambda_1 \leq \lambda_2 \leq \ldots \leq \lambda_n$), 
both depending on $\mathbf{q}$
such that 
$$ S(\mathbf{q}) = W(\mathbf{q}) \Lambda (\mathbf{q}) W(\mathbf{q})^T ,  $$
i.e. the columns of $W$ 
are the eigenvectors of $S(\mathbf{q})$, 
with the corresponding eigenvalues on the diagonal of $\Lambda$. 
By introducing the notation $\lambda_1(S(\mathbf{q}))$ to denote 
the smallest eigenvalue of $\Lambda(\mathbf{q})$,
we obtain the minimal value $\mathbf{p}_i^T S(\mathbf{q}_i) \mathbf{p}_i= \lambda_1(\mathbf{q}_i)$ 
when choosing $\mathbf{p}_i$ to be the (unit) eigenvector ($W_1(\mathbf{\mathbf{q}}_i)$)
corresponding to the smallest eigenvalue. 
As a consequence, the solution strategy boils down to steps in Algorithm~\ref{algorithm_mu_0}.

This choice of $P, Q$ and $B$ solves the constrained 
minimisation problem~\eqref{eq:functional_rsvd}.  
Notice that due to the fact that $P$ and $B$ matrices are determined after finding $Q$, this optimisation problem can essentially be translated into a search in the space of $Q$ matrices. 
Some illustrative numerical experiments are available at~\citep{code_theorem_5_6}.

\end{proof}
We conclude this section by giving a slightly more general 
version ($\mu \neq 0$) of the previous theorem, thus re-establishing the 
symmetry between $P$ and $Q$.  
\begin{theorem}[{\bf Regularised SVD, symmetric version}]
\label{thm_rsvd_mu_not_0}
Let $A$ be an $n\times m$ matrix of rank $r \leq \min(n,m)$.  For $k \leq r$,  let $P \in \R^{n\times k}$ and $Q \in \R^{m\times k} $ of rank $k$, 
while $B \in \R^{k\times k} $ diagonal 
(i.e.  $ B = diag(\beta_1, \beta_2, \ldots, \beta_k)  $).  Furthermore, 
for arbitrary non-zero integers $d$ and $g$
we introduce 
regularisation matrices $D \in \R^{d\times n}$, 
and $G \in \R^{g\times m}$,   
as well as  weights $\lambda, \mu \geq 0$. 
Finally, we introduce the short-hand notation 
$ L := D^T D \in \R^{n\times n}$ 
and $M := G^T G  \in \R^{m\times m}$ 
symmetric and positive-definite).
We are now in a position to define the following functional $F$ in the 
variables $P, Q$ and $B$:
\begin{equation}
F(P,Q, B) =  \n A - P B Q^T\n^2 + \lambda\, \n DP \n^2 
 + \mu\, \n GQ \n^2 
\label{eq:functional_1_full} 
\end{equation}
and the corresponding constrained optimisation problem: 
\begin{equation}
\min_{P,Q, B}F(P,Q, B) \quad \quad 
\mbox{subject to:}  \quad 
 \quad   Q^TQ = I_k, \quad  \,\n P_i\n=1, \,\,
 \mbox{$\forall i\in \{1,2,\ldots, k\}$ \, and  $B$ diagonal.}
\label{eq:functional_rsvd_1} 
\end{equation}
This problem is solved by the solution
specified in Algorithm~\ref{lag_mu_not_zero}.  
\end{theorem}

\begin{proof}
Using the notation introduced above and in Theorem~\ref{thm_rsvd_mu_0}, we see that 
$$ \n GQ \n^2 = \tr(Q^T M Q) 
= \sum_{i = 1}^k \mathbf{q}_i^T M  \mathbf{q}_i . $$
Hence, the functional (\ref{eq:functional_1_full}) can be recast as: 
\begin{equation}
F(P,Q) = \n A\n^2 + F_2(P,Q), \quad\quad 
\mbox{where} \quad \quad 
 F_2(P,Q) =
\sum_{i=1}^k  \mathbf{p}_i^T(\lambda L - A\mathbf{q}_i\mathbf{q}_i^TA^T)\mathbf{p}_i 
+ \mu\,\sum_{i = 1}^k \mathbf{q}_i^T M  \mathbf{q}_i  .
    \label{eq:functional_2_full}
\end{equation}
The minimum of each term in the first summation 
in $F_2$ 
is equal to the smallest 
eigenvalue $\lambda_1(S(\mathbf{q}_i))$. 
Finding the minimum for the constrained optimisation problem (\ref{eq:functional_rsvd_1}) therefore amounts to 
finding the minimum of the functional:
\begin{equation}
    \psi(Q) := \sum_{i = 1}^k  \left(\lambda_1(S(\mathbf{q}_i)) 
    + \mu \,\mathbf{q}_i^T M  \mathbf{q}_i) \right)  
\label{eq:min_psi}
\end{equation}
subject to the constraint $Q^T Q = I_k$. 
Therefore, the minimisation problem again calls for 
a minimisation in $Q$ space, as the optimal choice for $P$ (corresponding 
eigen-vectors) follows automatically. We therefore arrive 
at the following  Algorithm~\ref{lag_mu_not_zero}. Some illustrative numerical examples are available in~\citep{code_theorem_5_6}.
 
\end{proof}

\begin{algorithm}[H]
\SetAlgoLined
\KwIn{$A,\> k,\> \mu,\> \lambda, \> D, \> G$}
\KwOut{$P,\> B,\> Q $}
 Initialization\\
 \While{no convergence}{
  \begin{enumerate}
    \item Recall that for 
    any unit vector $\mathbf{q} \in \R^m$ we define $S(\mathbf{q}) = \lambda L - Aqq^TA^T$.  Since this is a symmetric 
    $n\times n$
    matrix, it has a complete set of eigenvectors and corresponding eigenvalues. 
    Denote the smallest eigenvalue of each $S(\mathbf{q}_i)$  as 
    $\lambda_1(S(\mathbf{q}_i))$.
    
    \item For a given  $m\times k$ matrix $Q = [\mathbf{q}_1, \mathbf{q}_2, \ldots, \mathbf{q}_k]$ 
    (with orthonormal columns: $Q^TQ = I_k$) 
    compute the functional: 
    $$ \psi(Q) := \sum_{i = 1}^k  \left(\lambda_1(S(\mathbf{q}_i)) 
    + \mu \,\mathbf{q}_i^T M  \mathbf{q}_i) \right)
    $$
    and use gradient descent (on the compact \textit{torus domain}, 
    see section~\ref{sct:comp_aspects}) to 
    find the minimum. 
   
    \item  For each $\mathbf{q}_i$ as determined above, take $\mathbf{p}_i$  
    to be the eigenvector $W_1(\mathbf{q}_i)$ corresponding 
    to the smallest eigenvector $\lambda_1(S(\mathbf{q}_i))$. 
    Construct the $n\times k$ matrix 
     $P = [\mathbf{p}_1, \mathbf{p}_2, \ldots, \mathbf{p}_k]$. 
     
    \item Finally, set $B = diag(\beta_1, \ldots, \beta_n)$ where  $\beta_i = (P^TAQ)_{ii}$.
\end{enumerate}

 }
 \caption{Proposed RSVD method ($\mu \neq  0$)}
 \label{lag_mu_not_zero}
\end{algorithm}


\section{Computational Aspects}
\label{sct:comp_aspects}
\subsection{Gradient Descent on the Unitary Domain}
From Algorithm~\ref{lag_mu_not_zero} it becomes clear that 
the full regularisation problem 
can be reduced to the simpler 
constrained minimisation problem 
detailed in eq.~(\ref{eq:min_psi}).  Since 
the $\psi$-functional is smooth 
on a compact domain,  
this minimum is guaranteed to 
exist and one can use gradient descent to 
locate it. However, gradient descent 
needs to respect the constraint $Q^T Q = I_k$. This can 
be achieved by applying orthogonal transformations 
to the current $Q$ matrix, as this will 
preserve orthonormality. Specifically, recall all 
orthogonal $m\times m$ matrices with determinant equal to 1 
(rather than $-1$),  
constitute a multiplicative group denoted as
$\mathcal{SO}(m)$ and formally  defined as: 

\begin{equation*}
    \mathcal{SO}(m) = \left\{  R \in \R^{m\times m} \given 
R R^T = I_m = R^T R, \quad \mbox{and} 
\quad \det (R) = 1
\right\}
\end{equation*}
It is then straightforward to check that 
for any  $R \in \mathcal{SO}(m)$, it holds that if 
$\bar{Q} = RQ$, 
the condition $Q^T Q = I_k$ implies that 
 $\bar{Q}^T\bar{Q} = I_k$.  It therefore follows 
 that we can generate the ``\textit{infinitesimal 
 variations}" needed to compute the gradient 
 $\nabla_Q \> \psi(Q)$ by applying  ``sufficiently 
 small'' orthogonal matrices to the 
 current value of $Q$. 
 More precisely, we draw on the fact that 
 $\mathcal{SO}(m)$ is actually a Lie-group \citep{iserles2000lie}
 and that therefore each $R \in \mathcal{SO}(m)$ 
 can be generated by exponentiating 
 an element from its Lie-algebra 
 $  so(m) = \left\{ K \in \R^{m \times m} \given K^T = -K\right\}$ (the skew-symmetric matrices):   
\begin{equation*}
    R = \exp(tK) \equiv I_m + tK + \frac{1}{2!}t^2K^2 +\ldots 
+  \frac{1}{n!}t^n K^n + \ldots  \quad\quad 
\mbox{(with $K^T = -K$)}
\end{equation*}
By choosing $t$ sufficiently small, one obtains an orthogonal 
transformation that is close to the identity $I_m$. 
Furthermore, it suffices to restrict the variations to 
orthogonal transformations that result from exponentiating 
a basis 
for the space of skew-symmetric matrices. 
Such a basis is provided by  
the $m(m-1)/2$
skew-symmetric matrices $K_{ij}$ 
(where $1 \leq i < j \leq m$) for which the 
matrix element $k,\ell$ is given by: 
\begin{equation*}
    K_{ij}(k,\ell) = \left\{
\begin{array}{rcl}
    1 & \mbox{if} & k = i, \, \ell = j  \\
    -1 & \mbox{if} & k = j,\, \ell = i  \\
    0 & & \mbox{otherwise}
\end{array}
\right.
\end{equation*}
Given the current value $Q_0$, 
we construct nearby 
values for $Q$  by looping over $K_{12}, 
K_{13}, K_{23}, \ldots etc$ 
and constructing the corresponding orthogonal matrices  
$R_{12}(t) = \exp(tK_{12}),\ldots, etc$. 
Denoting these ``\emph{infinitesimal}'' rotation 
matrices as $R_{\alpha}$ (where 
$\alpha = 1,\ldots, m(m-1)/2$),  
we see that the partial derivatives with 
respect to these rotations can be estimated 
as: 
\begin{equation*}
    \frac{ \partial \psi(Q)}{\partial R_\alpha} 
\approx \frac{\psi(R_\alpha(t) Q_0) - \psi(Q_0)}{t} 
\quad \quad \quad \quad \mbox{(for $t$ sufficiently small). }
\end{equation*}
From these results we can select the 
infinitesimal rotation that results in the 
steepest descent.  

Since computing  $\psi$ is computationally expensive 
(it requires determining eigenvalues) a viable alternative 
to computing the gradient, is random descent: 
generate random rotations 
(by exponentiating random skew matrices) and check
whether they result in a lower $\psi$-value.  As 
soon as one is found, proceed in that direction, and 
repeat the process.

\subsection{Illustrative example:  Smoothing a noisy matrix}

As common in the literature e.g.,~\citep{jin2015low, he2019graph, gavish2014optimal}, we start from the assumption that 
the $n\times m$ data matrix $A$ has a relatively smooth underlying structure 
that is corrupted by noise: 
$$  A = \mathbf{u} \mathbf{v}^T + \tau Z,  $$
where the $n\times m$ matrix $Z $ has independent standard 
normal entries, and $\tau$ controls the size of the noise. 

To recover the underlying "signals"  $ \mathbf{u}$ 
and  $\mathbf{v}$,  
we minimise the SVD-type regularisation 
functional~\eqref{eq:functional_1_full} 
where the smoothness of the result is enforced by 
using regularisation matrices $D$ and $F$ that 
extract the second derivative, i.e. 
$$D = F=
 \begin{bmatrix}
  -1 & 1 & 0 & & \cdots& & 0 \\
  1 & -2 & 1 & 0& \cdots&  & 0 \\
  0 & 1 & -2 & 1& 0 &\cdots & 0 \\
  0 & 0 & 1 & -2& 1 &\cdots & 0 \\
  0 & & \ddots&\ddots &\ddots&\ddots&0  \\
   \vdots& &\cdots &0&1&-2&1 \\
  0 &  &\cdots&    &0&1 & -1
 \end{bmatrix}
$$
A typical result 
for a rank-1 ($k=1$) approximation 
is depicted in Figure~\ref{fig:rsvd_exp_2}, and 
compared to the standard SVD solution. This illustrative example is available in~\citep{num_exp_sec}.


\begin{figure}[]
    \centering
   
      \includegraphics[width=0.45\textwidth]{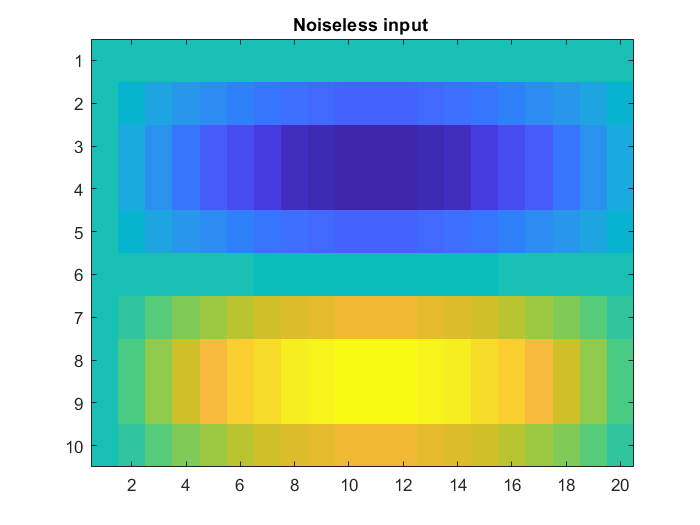}
    \includegraphics[width=0.45\textwidth]{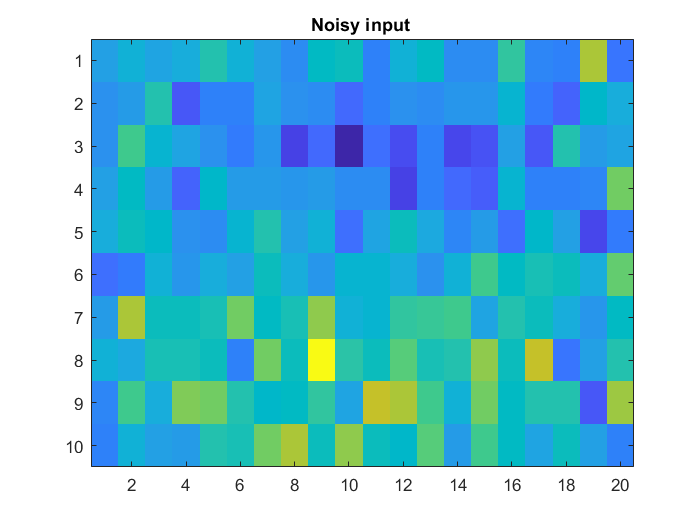}
     \includegraphics[width=0.45\textwidth]{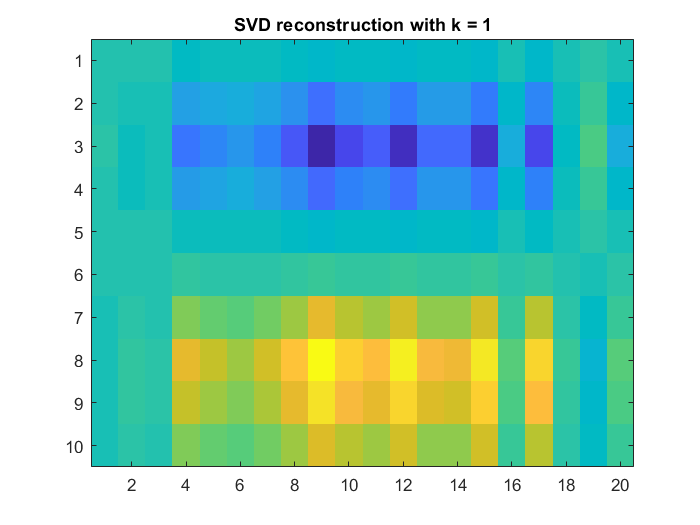}
      \includegraphics[width=0.45\textwidth]{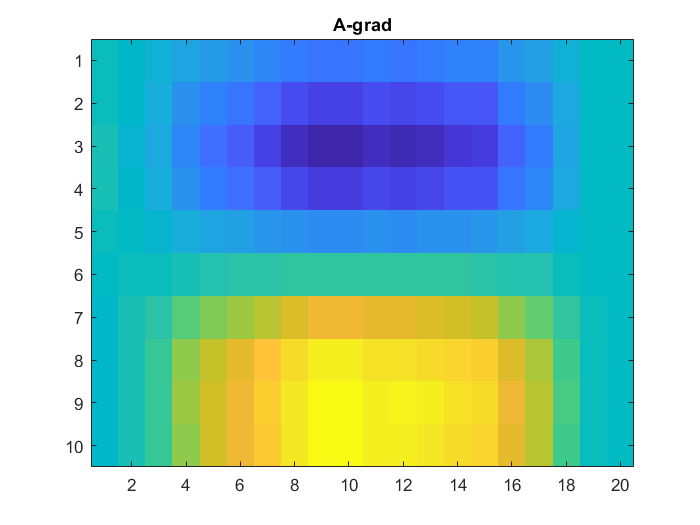}
       \includegraphics[width=0.45\textwidth]{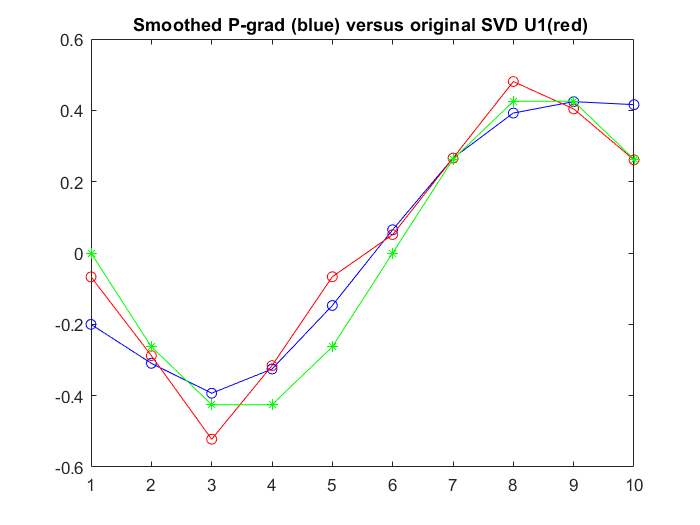}
    \includegraphics[width=0.45\textwidth]{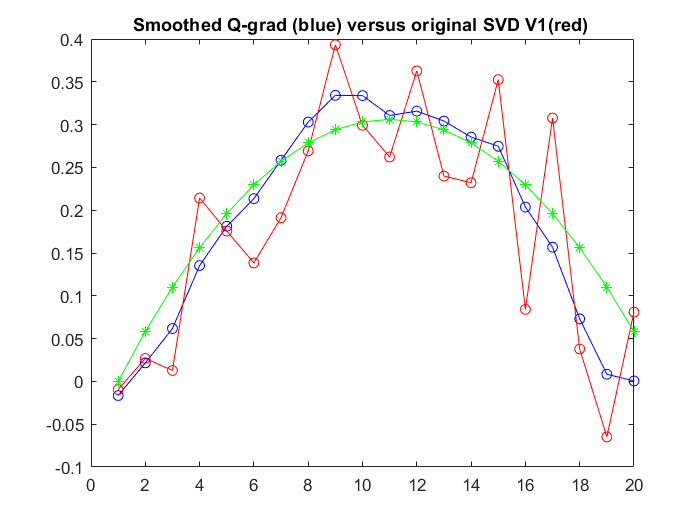}
    \caption{ 
    Reconstruction of noisy matrix based on RSVD. 
    Top left: noise-less rank-1 matrix $\mathbf{u} \mathbf{v}^T ,$ (image) ,  top right: noisy input image 
     $\mathbf{u} \mathbf{v}^T + \tau Z$
    (high noise level), Middle left: 
    standard rank-1 SVD reconstruction, middle right: RSVD 
    reconstruction (D and F are 2nd deriv matrices.  
    weight parameters $\lambda = \mu = 1.5$).
    Bottom: comparison of standard SVD $U(:,1)$ (red) versus $P$ (blue), 
    and $V(:,1)$ (red) (left) versus $Q$ (blue) (right). The actual  $\mathbf{u}$  and   $\mathbf{v}$ for the noiseless 
    input signal are drawn in green. 
    }
    \label{fig:rsvd_exp_2}
\end{figure}

\section{Conclusions and Future Research}
\label{sec:conclusion}

Singular Value Decomposition (SVD) and Principal Component 
Analysis (PCA) are important matrix factorisation techniques that underpin numerous applications. However, it is well-known that disturbances in the input (noise, outliers or  missing values) have a significant effect on the outcome. 
For that reason we investigate regularisation in two 
different but related versions of the factorisation, 
and detail the solution algorithms. 

An important topic for further research 
would be to find ways in which the gradient descent procedure 
in Algorithms~\ref{algorithm_mu_0} and 
 \ref{lag_mu_not_zero}
can be accelerated by taking advantage of the fact that 
the functional is very smooth and locally approximately 
quadratic. It would also be useful to derive some estimates 
for appropriate values for the weights $\lambda$ and 
$\mu$ in terms of noise characteristics corrupting the 
underlying signal. Finally, although the $P$ matrix in 
algorithm  \ref{lag_mu_not_zero} has unit-length columns, 
we were not able to prove that these columns are 
also orthogonal ($P^TP = I$) as is the case in standard SVD. 
In fact, numerical experiments seem to indicate that 
such a constraint is not compatible with minimisation 
of the functional.  This requires further theoretical 
elucidation.


\section*{Acknowledgment}
The authors gratefully acknowledge partial support by the Dutch NWO 
ESI-Bida project NEAT (647.003.002). 

\bibliographystyle{elsarticle-num}
\bibliography{main}
%
%

\end{document}